\documentclass{article}

\usepackage[final]{neurips_wrl2022}

\usepackage{amsfonts}
\usepackage{amsmath}
\usepackage{amssymb}
\usepackage{amsthm}
\usepackage{bm}
\usepackage{caption}
\usepackage{graphicx}
\usepackage{hyperref}
\usepackage[bb=dsserif]{mathalpha}
\usepackage[capitalize, noabbrev]{cleveref}
\usepackage{xcolor}
\usepackage[utf8]{inputenc} 
\usepackage[T1]{fontenc}    
\usepackage{url}            
\usepackage{booktabs}       
\usepackage{nicefrac}       
\usepackage{microtype}      

\newtheorem{theorem}{Theorem}

\DeclareMathOperator{\vect}{vec}
\DeclareMathOperator{\relu}{ReLU}

\newcommand{\matr}[1]{\bm{#1}}
\newcommand{\onehat}{\hat{\mathbb{1}}_{N_s,N_A}^\intercal}

\graphicspath{{images/}}

\title{Imitating careful experts to avoid catastrophic events}

\author{%
  Jack R. P.~Hanslope \\
  Department of Computer Science \\
  University of Bristol \\
  Bristol, UK \\
  \texttt{jack.hanslope@bristol.ac.uk} \\
  \And
  Laurence Aitchison \\
  Department of Computer Science \\
  University of Bristol \\
  Bristol, UK \\
  \texttt{laurence.aitchison@bristol.ac.uk} \\
}

\begin{document}

\maketitle

\begin{abstract}
RL is increasingly being used to control robotic systems that interact closely with humans.
This interaction raises the problem of safe RL: how to ensure that a RL-controlled robotic system never, for instance, injures a human.
This problem is especially challenging in rich, realistic settings where it is not even possible to clearly write down a reward function which incorporates these outcomes.
In these circumstances, perhaps the only viable approach is based on IRL, which infers rewards from human demonstrations.
However, IRL is massively underdetermined as many different rewards can lead to the same optimal policies; we show that this makes it difficult to distinguish catastrophic outcomes (such as injuring a human) from merely undesirable outcomes.
Our key insight is that humans do display different behaviour when catastrophic outcomes are possible: they become much more careful.
We incorporate carefulness signals into IRL, and find that they do indeed allow IRL to disambiguate undesirable from catastrophic outcomes, which is critical to ensuring safety in future real-world human-robot interactions.
\end{abstract}

\section{Introduction}
Industry is increasingly experimenting with removing traditional safety barriers between robots and humans in favour of close collaboration \citep{galin2020cobots,moreno2020autonomous,simoes2019drivers,michalos2022human,hjorth2022human}.
However, such close collaboration raises serious risks, as industrial robots can and sometimes do seriously injure humans \citep{henley2017finger}.
This raises a question: can we design RL driven robotic control systems that never take catastrophically bad actions such as injuring a human?
This is particularly challenging in a rich environment where it is difficult to rigidly define or simulate all catastrophic outcomes.
In such circumstances, perhaps the only viable approach is some form of imitation learning (a generic term for learning by observing an expert)\citep{Ho_Ermon_2016,Paine_Gomez_etal_2018,Peng_Abbeel_Levine_vandePanne_2018,Peng_Kanazawa_Toyer_Abbeel_Levine_2022}. 
Specifically, we consider inverse reinforcement learning (IRL), in which we infer the reward function driving expert behaviour \citep{Russell_1998, ng2000algorithms, ziebart2008maximum,Finn_Levine_Abbeel_2016, Peng_Kanazawa_Toyer_Abbeel_Levine_2022}.
%


%

However, IRL has one fundamental problem: that many reward functions are usually compatible with the observed behaviour \citep{ng2000algorithms, ramachandran2007bayesian}.
Often, it will be impossible to distinguish between catastrophic events and events that are merely undesirable. 

In contrast, humans and other animals are able to learn to avoid catastrophic events with a very small number of samples (often only one) and observing only e.g.\ a parents or teacher's response to a \textit{potential} catastrophic event, without observing the catastrophic event itself.
For instance, an infant animal will learn to be afraid of a predator if they merely observe their parent's fearful responses to that predator (they do not need to observe another animal actually being predated) \citep{mineka1993mechanisms,askew2008vicarious,dunne2013vicarious,reynolds2018reductions,marin2020vicarious}. 
Alternatively, human operators of robots will be more careful when there is a risk of catastrophic outcomes: for instance, they might slow down when their robot is operating in the vicinity of a human.
We should be able to use these carefulness signals to establish when the human operator believes there is a risk of catastrophic outcomes.

We design an IRL framework incorporating carefulness signals and find that it can distinguish between catastrophic and undesirable outcomes which are in practice indistinguishable without considering carefulness.
We begin with a toy experiment in a gridworld.
The gridworld has a cliff along one side and falling off the cliff represents the catastrophic event.
We initially run simulations in this gridworld with computer experts that know the optimal policy to be followed.
We then run the same simulations again with a human player with a good understanding of the environment and optimal policy.
We see that, in both scenarios, our framework including carefulness enables a better understanding of the severity of falling off of the cliff.

\section{Related Work}
There is a considerable body of work in safe RL \citep{alshiekh2018safe, saunders2017trial,NIPS2017_766ebcd5,pmlr-v70-achiam17a,garcia2012safe}, safe imitation learning \citep{menda2017dropoutdagger, menda2019ensembledagger} and safe IRL \citep{pmlr-v87-brown18a, Brown_Niekum_2018}.
However, none of this work uses the key insight that human carefulness might be used to distinguish between catastrophic and merely undesirable actions.

\section{Methods}

\textbf{Gridworld Environment.}
We use a four by six gridworld \citep{sutton2018reinforcement} environment with a cliff along one edge and a goal-state at the end of the cliff.
Entering any of the cliff states represents a catastrophic event and results in the termination of the episode and a very large negative reward being received.
Reaching the goal-state also terminates the episode and gives a medium positive reward.
The remaining states have a reward of 0.
There is a small cost for each action.
The layout of the environment is shown in \cref{fig:grid_world_layout}.

In the simplest setting of the environment, there are four actions, corresponding to the four directions the agent can move in.
Movement in the environment is stochastic, meaning that there is a chance that the agent moves in a random direction rather than the direction they have chosen.
The cost for all actions is the same.

In more complex versions of the environment, there are more actions.
The actions are tuples of a direction and an amount of carefulness.
Choosing a higher carefulness results in the agent being more likely to follow the specified action, but results in a higher movement cost.
The movement cost is linear in the level of carefulness and the probability of following the specified action is $1-2^{-c}$, where $c$ is the level of carefulness.

\textbf{Simple IRL agent.}
\label{section:method:simple_irl_agent}
Given a policy $\pi$ (or rollouts from that policy), our goal is to find a reward, $R$, under which $\pi$ is optimal.
We use constrained optimization.
In particular, the optimality of $\pi$ imposes constraints on $Q$, and because we can write $Q$ in terms of $R$, we therefore get constraints on $R$.
We give the specific form of these constraints in the Appendix (extensions of those in \citet{ng2000algorithms}) and solve for $R$ using linear programming \citep{cvxopt}.
%
%
%
However, remember that one of the key issues with IRL is distinguishing between these valid reward functions.
Following our extension of \citet{ng2000algorithms}, we maximise
\begin{equation}
  \label{equ:old_objective_function}
  \sum_{s \in S} \left\{ Q^\pi(s, \pi(s)) - \max_{a \in A \backslash \{\pi(s)\}} Q^\pi(s, a) \right\}- \lambda \vert\vert \matr{R} \vert\vert_1
\end{equation}
This encourages reward functions that make the action chosen by the optimal policy better than the next best action by as large a margin as possible.
The second term ($\lambda || \matr{R} ||_1$) encourages the reward function to be sparse (have many zeros).
\cref{thm:r_valid_equiv} in the Appendix describes how to write this objective purely in terms of $\matr{R}$.

\textbf{Expressing the reward in terms of the state reward.}
In our environments,  the reward can be expressed as the sum of an action-reward and a state-reward,
\begin{equation}
\label{equ:r_state_plus_action}
R(s, a) = R_A(a) + R_S(s)
\end{equation}
This form for the reward greatly helps in solving the IRL problem, as we only need to find $|S| + |A|$ unknown rewards; without this decomposition we would need to find $|S||A|$ unknown rewards.


\textbf{Game and Loss IRL agent.}
\label{section:method:game_loss_based_irl_agent}
To test the IRL agent on human data, we developed an interactive, human-playable version of the gridworld.
The player is able to use the arrow keys to move around the gridworld and must try to reach the goal state without falling off the cliff.
As before, the environment is stochastic: there is a chance that the avatar will not move in the direction the player specifies.
The player is able to counteract this by holding down the movement key for longer.
As they do this, the bar at the side representing carefulness will fill up and the arrow showing the direction of travel will also grow.
The player starts the episode with a score of 200 and being more careful is more costly (the costs are shown next to the side bar).
The layout of the game is shown in \cref{fig:game}

\begin{figure}
\centering
\begin{minipage}[t]{0.45\textwidth}
  \centering
  \includegraphics[height=1.5in]{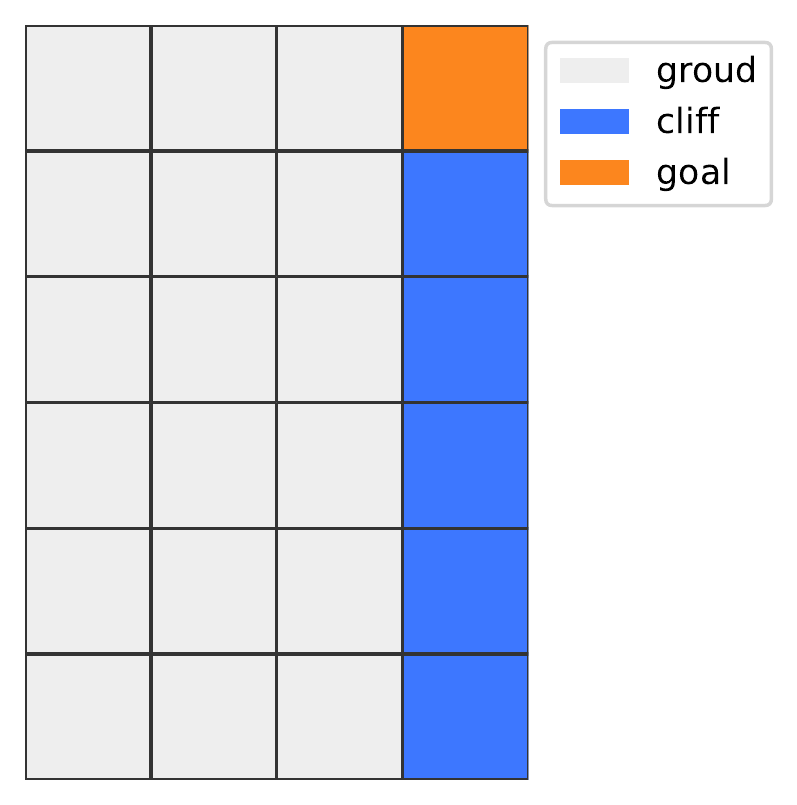}
  \captionof{figure}{Layout of the gridworld}
  \label{fig:grid_world_layout}
\end{minipage}%
\hfill
\begin{minipage}[t]{0.45\textwidth}
  \centering
  \includegraphics[height=1.5in]{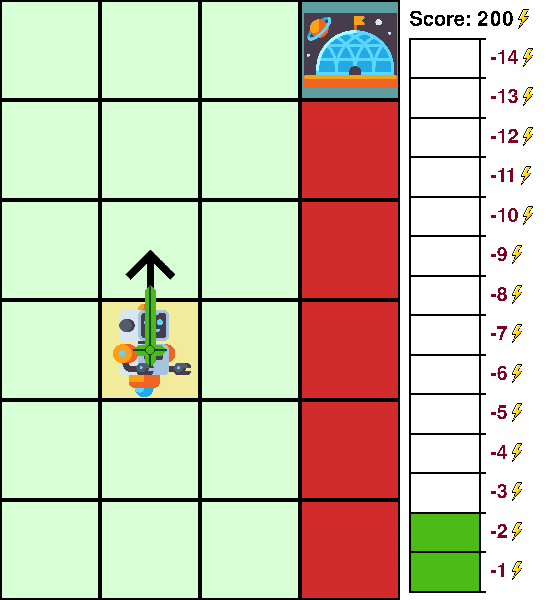}
  \captionof{figure}{Layout of the game}
  \label{fig:game}
\end{minipage}
\end{figure}

Unsurprisingly, when a human plays the game, they do not not follow the exact optimal policy. 
This is problematic as the simple IRL agent above assumes that the optimal policy is known and deterministic.
To resolve this issue, we developed another method which we call the loss IRL agent,
\begin{align}
\mathcal{L}(R_S) = \sum_{s \in S}\exp\left( \relu \left( \sum_{a \in A} \pi(s, a) \max_{a' \in A \backslash \{a\}}Q(s, a') - Q(s, a)\right) \right)
\end{align}
We will minimise the loss, $\mathcal{L}(R_S)$.
We would expect that $Q$ would be highest for the action most preferred by the policy.
When the most preferred action for a given state results in the largest $Q$ value, the term inside the $\relu$ will be 0, otherwise, it will be positive.
%

\section{Experimental results}

\textbf{MaxEnt baselines} As a baseline, we use a MaxEnt method without carefulness from \citet{pmlr-v139-kim21c}.
We choose this method because it is guaranteed to recover the reward in the limit of infinite data, even without carefulness.
The method has deterministic transitions with policy
\begin{align}
\pi(a | s) &= \frac{1}{Z(s)} e^{\beta Q(s, a)}
\end{align}
where $\beta $ is a hyperparameter and $Z(s)$ ensures that the distribution over actions normalizes to $1$.
\citet{pmlr-v139-kim21c} guarantees the existence of a unique optimal reward function for this setting.
We generated 10000 rollouts from an agent in this environment and recover a reward function which is shown in \cref{fig:learned_reward_computer_rollouts_benchmark}.
Thus, while this reward function accurately captures the existence of both the cliff and the goal state, it massively underestimates the severity of the penalty for falling off the cliff, and this likely indicates the need for far more samples to accurately estimate the rewards.

\textbf{Simple IRL agent.}
This is for the agent described in \cref{section:method:simple_irl_agent}.
We solve the MDP using value-iteration to get an optimal policy for the setting with 14 levels of carefulness.
The policy is shown in \cref{fig:optimal_policy}.
As we would expect, the optimal actions when closer to the cliff are to be more careful and often to move away from the cliff.
\begin{figure}
\centering
\begin{minipage}[t]{0.3\textwidth}
  \centering
  \includegraphics[width=0.9\linewidth]{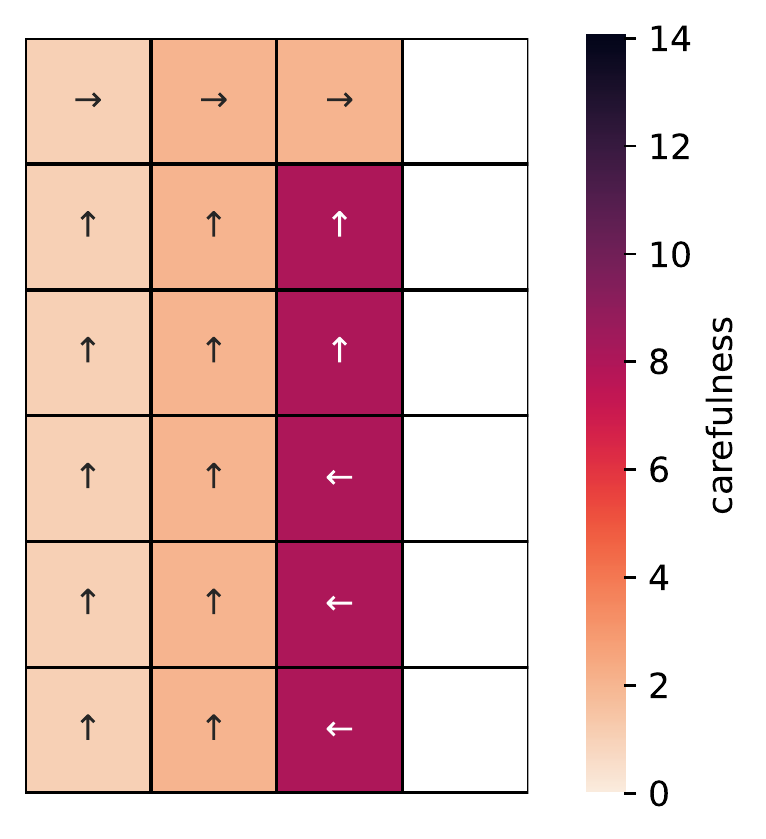}
  \captionof{figure}{Optimal policy for the setting with 14 levels of carefulness}
  \label{fig:optimal_policy}
\end{minipage}%
\hfill
\begin{minipage}[t]{0.3\textwidth}
  \centering
  \includegraphics[width=0.9\linewidth]{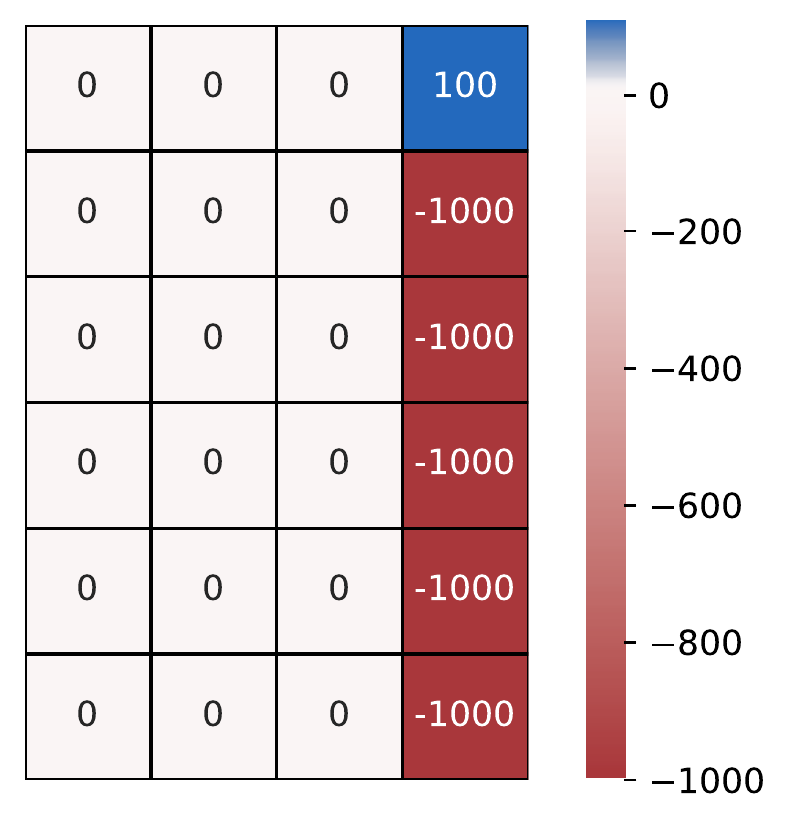}
  \captionof{figure}{True reward}
  \label{fig:true_reward}
\end{minipage}
\hfill
\begin{minipage}[t]{0.3\textwidth}
  \centering
  \includegraphics[width=0.9\linewidth]{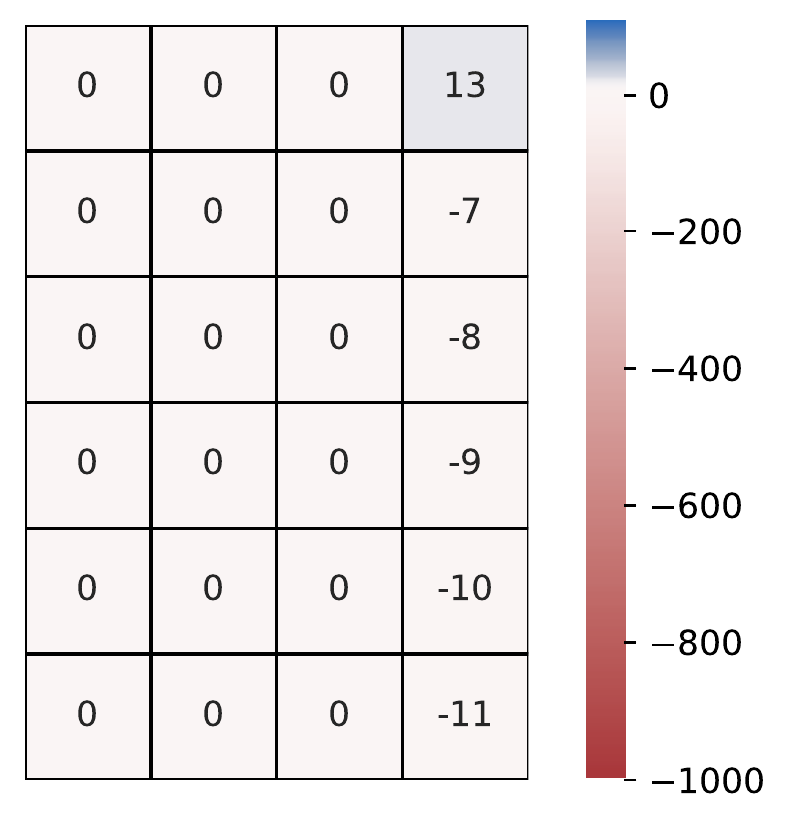}
  \captionof{figure}{Learned reward using computer generated rollouts in the benchmark setting}
  \label{fig:learned_reward_computer_rollouts_benchmark}
\end{minipage}
\end{figure}
We generated 100 rollouts with the agent starting in a random ``ground'' state (see \cref{fig:grid_world_layout}) and then following the optimal policy shown in \cref{fig:optimal_policy}.
We then use these rollouts to train the simple IRL agent using a maximum reward of 1000 and we used $\lambda=0$.
By leveraging carefulness, we are able to learn the severity of the reward function; see \cref{fig:learned_reward_computer_rollouts}.

\begin{figure}
\centering
\begin{minipage}[t]{0.3\textwidth}
  \centering
  \includegraphics[width=0.9\linewidth]{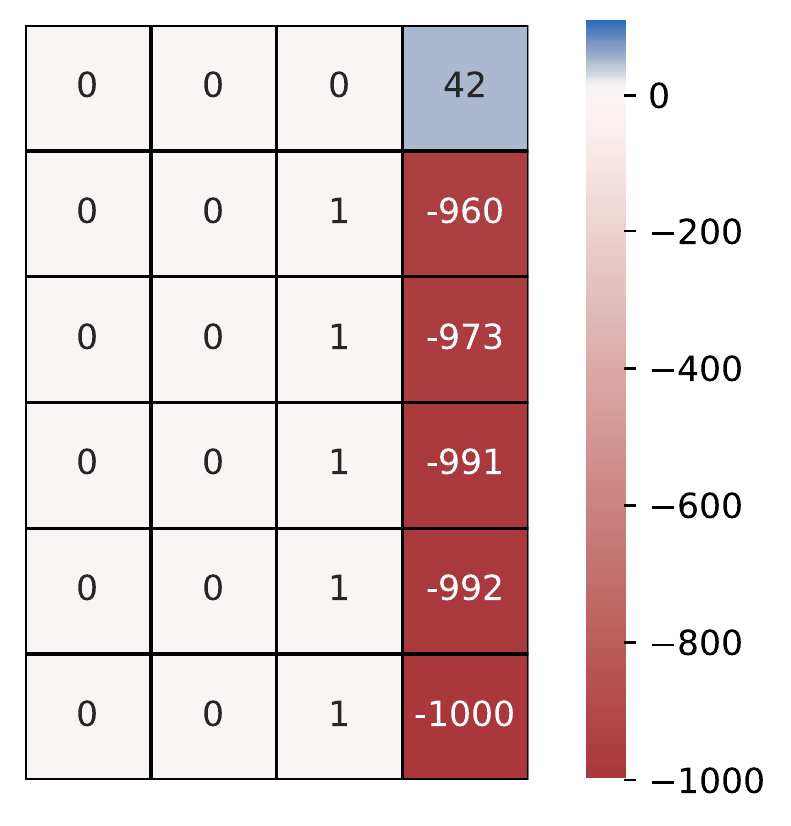}
  \captionof{figure}{Learned reward using computer generated rollouts}
  \label{fig:learned_reward_computer_rollouts}
\end{minipage}%
\hfill
\begin{minipage}[t]{0.3\textwidth}
  \centering
  \includegraphics[width=0.9\linewidth]{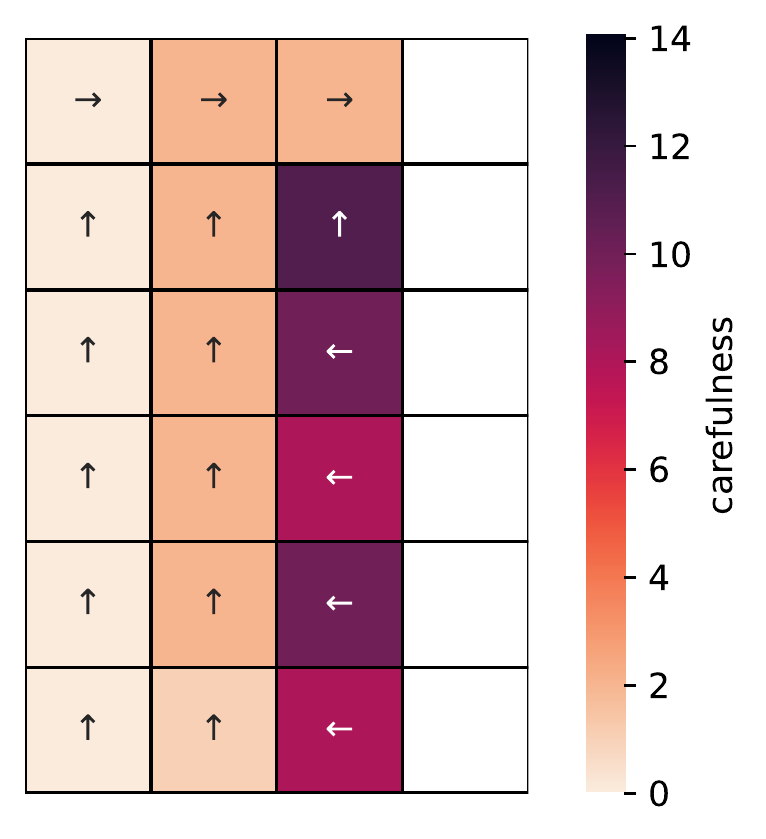}
  \caption{Policy followed by the human}
  \label{fig:human_policy}
\end{minipage}%
\hfill
\begin{minipage}[t]{0.3\textwidth}
  \centering
  \includegraphics[width=0.9\linewidth]{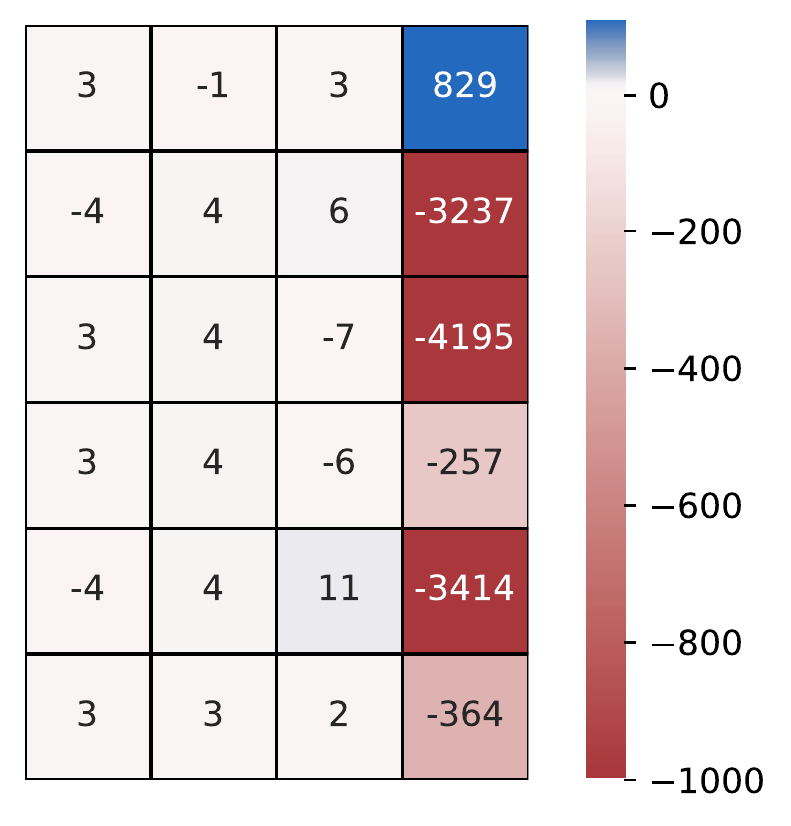}
  \caption{Learned reward from human rollouts}
  \label{fig:human_reward}
\end{minipage}
\end{figure}

\textbf{Loss IRL agent.}
An expert human (one of the authors) played the game for 10 rollouts (vs 10,000 optimal rollouts for the baseline).
Note that some states will have been visited on multiple occasions and that different actions may have been selected on each visitation and so the policy is not deterministic.
The most commonly chosen action for each state is shown in \cref{fig:human_policy}.
An agent as described in \cref{section:method:game_loss_based_irl_agent} was trained on the rollouts and the reward obtained is shown in \cref{fig:human_reward}; compare to the true reward in \cref{fig:true_reward} and the benchmark reward in \cref{fig:learned_reward_computer_rollouts_benchmark}.
While there is considerable variability in the inferred rewards, likely due to only having 10 rollouts, we have achieved our aim, which was to capture the severity of the penalty for falling off the cliff.


\section{Conclusion and future work}
We have shown that IRL can distinguish between catastrophic and merely undesirable rewards if it takes account of human carefulness.
We validated this idea on a simple gridworld task where the human explicitly inputs a degree of carefulness.
In future work, it will be important to extend these ideas to more complex settings (such as robotic simulators) and to understand how to infer carefulness from actual human actions (e.g.\ by using the speed of actions as a proxy).


\bibliographystyle{iclr2023_conference}
\bibliography{references.bib}

\appendix 

\section{Background}
A (finite) Markov decision process (MDP) is defined as the tuple $(S, A, T, \gamma, R)$ where
\begin{itemize}
  \item $S$ is a finite set of $N_S$ states
  \item $A$ is a finite set of $N_A$ actions
  \item $T(s, a, s')$ is the probability of transitioning from state $s$ to state $s'$ when undertaking action $a$
  \item $\gamma \in [0, 1)$ is the discount factor
  \item $R(s, a)$ is the reward received from transitioning out of state $s$ via action $a$
\end{itemize}

A policy is defined as a function $\pi : S \times A \rightarrow [0, 1]$.
A value function for a given policy is a map $V^\pi : S \rightarrow \mathbb{R}$ and represents the expected discounted reward when following policy $\pi$ from a given state.
It is recursively given by
\begin{equation}
V^\pi(s) = \sum_{a \in A} \pi(s, a) \left( R(s, a) + \gamma \sum_{s'\in S} T(s, a, s')V^\pi(s')\right)
\end{equation}
A related quantity is the action-value function, $Q^\pi : S \times A \rightarrow \mathbb{R}$. $Q^\pi(s, a)$ is the expected discounted reward when starting from state $s$, immediately following action $a$ and then following policy $\pi$.
It is given in terms of $V^\pi$ by
$$
Q^\pi(s, a) = R(s, a) + \gamma \sum_{s' \in S} T(s, a, s') V^\pi(s')
$$

\section{Proofs for IRL agent}

\begin{theorem}
  \label{thm:v_in_terms_r}
  $$
  \matr{V}^\pi = (\matr{I}_{N_S} - \gamma \matr{T} ^ \pi) ^{-1} \hat{\matr{\pi}} \vect(\matr{R})
  $$
  Where $N_S:=\vert S \vert$, $\matr{T}^\pi$ is the $N_S$ by $N_S$ matrix giving the probability of transitioning from one state to another when following the optimal policy, and $\hat{\matr{\pi}}$ is the $N_S$ by $N_SN_A$ matrix such that $\hat{\matr{\pi}} \vect(\matr{R}) = \matr{R}^\pi$ and $\matr{R}^\pi$ is the expected immediate reward for transitioning whilst following $\pi$.
\end{theorem}

\begin{proof}
  We have that, $\forall s \in S$
  \begin{align*}
    V^\pi(s) &= \sum_{a \in A} \pi(s, a)\left(R(s, a) + \gamma \sum_{s'\in S} T(s, a, s') V^\pi(s')\right) \\
    V^\pi(s) &= \sum_{a \in A} \pi(s, a)\left( \matr{R}_{s, a} + \gamma \matr{T}_{s, a} \matr{V}^\pi\right)
  \end{align*}
  and so
  \begin{align*}
    \matr{V}^\pi &= \matr{R}^\pi + \gamma \matr{T}^\pi \matr{V}^\pi \\
    (\matr{I} - \gamma\matr{T}^\pi) \matr{V}^\pi &=\matr{R}^\pi \\
    \matr{V}^\pi &= (\matr{I} - \gamma \matr{T}^\pi)^{-1} \hat{\matr{\pi}}\vect{\matr{R}}
  \end{align*}
\end{proof}

\begin{theorem}
  \label{thm:q_in_terms_r}
  $$
  \matr{Q}^\pi = \left( \hat{\matr{I}} + \gamma \matr{T} (\matr{I}_{N_S} - \gamma \matr{T}^\pi)^{-1} \hat{\matr{\pi}} \right) \vect(\matr{R})
  $$
  where $\hat{\matr{I}}$ is $\matr{I}_{N_S N_A}$ reshaped to be $N_S$ by $N_A$ by $N_SN_A$
\end{theorem}

\begin{proof}
  \begin{align*}
    Q^\pi(s, a) &= R(s, a) + \gamma \sum_{s' \in S} \matr{T}_{s, a}(s') V^\pi(s') \\
    \matr{Q}^\pi &= \matr{R} + \gamma \matr{T} \matr{V}^\pi \\
    \matr{Q}^\pi &= \matr{R} + \gamma \matr{T}(\matr{I}_{N_S} - \gamma \matr{T}^\pi )^{-1} \hat{\matr{\pi}} \vect({\matr{R}}) \\
    \matr{Q}^\pi &= \left( \hat{\matr{I}} + \gamma \matr{T} (\matr{I}_{N_S} - \gamma \matr{T}^\pi)^{-1} \matr{\hat{\pi}}\right) \vect{(\matr{R})}
  \end{align*}
\end{proof}

\begin{theorem}
  \label{thm:r_valid}
  For a reward function $\matr{R}$ to be valid, we require that
  $$
  \left\{\left[ \widetilde{\matr{I}}_{N_S} + \gamma (\widetilde{\matr{T}}^\pi - \matr{T}) (\matr{I}_{N_S} - \gamma \matr{T}^\pi ) ^{-1} \right] \hat{\matr{\pi}} - \hat{\matr{I}} \right\} \vect(\matr{R}) \succeq 0
  $$
  where
  \begin{itemize}
    \item For a matrix $\matr{X}$ of dimension $N_S$ by $N_S$, $\widetilde{\matr{X}}$ is the matrix obtained by stacking $N_A$ copies of $\matr{X}$ giving a matrix of dimension $N_S$ by $N_A$ by $N_S$
    \item For matrices $\matr{X}$ of dimension $n$ by $m$, and $c \in \mathbb{R}$, $\matr{X} \succeq c$ means that for all $0 \leq i\leq n$ and $0 \leq j \leq m$, $\matr{X}_{i, j} \geq c$
  \end{itemize}
\end{theorem}

\begin{proof}
  The policy $\pi$ is optimal if and only if $\forall s \in S$ and $\forall a \in A$
  \begin{align*}
    & Q^\pi(s, \pi(s)) \geq Q^\pi(s, a') \\
    &\qquad   \Leftrightarrow \matr{R}_{s, \pi(s)} + \gamma \left(\matr{T}^\pi \matr{V}^\pi \right)_s \geq \matr{R}_{s, a} + \gamma \left( \matr{T}_{:, a, :} \matr{V}^\pi \right)_s  \\
    &\qquad  \Leftrightarrow \matr{R}^\pi + \gamma \matr{T}^\pi \matr{V}^\pi \succeq \matr{R}_{:, a} + \gamma \matr{T}_{:, a, :} \matr{V}^\pi  \\
    &\qquad \Leftrightarrow \hat{\matr{\pi}} \vect{(\matr{R})} - \matr{R}_{:, a}  \succeq \gamma (\matr{T}_{:, a, :} \matr{V}^\pi - \matr{T}^\pi \matr{V}^\pi ) \\
    &\qquad \Leftrightarrow \hat{\matr{\pi}} \vect{(\matr{R})} - \matr{R}_{:, a}  \succeq \gamma(\matr{T}_{:, a, :} - \matr{T}^\pi) \matr{V}^\pi \\
    &\qquad \Leftrightarrow \hat{\matr{\pi}} \vect{(\matr{R})} - \matr{R}_{:, a}  \succeq \gamma (\matr{T}_{:, a, :} - \matr{T}^\pi)(\matr{I}_{N_S} - \gamma \matr{T}^\pi)^{-1} \hat{\matr{\pi}} \vect (\matr{R}) \\
    &\qquad  \Leftrightarrow \hat{\matr{\pi}} \vect (\matr{R}) - \gamma (\matr{T}_{:, a, :} - \matr{T}^\pi)(\matr{I}_{N_S} - \gamma \matr{T}^\pi)^{-1} \hat{\matr{\pi}} \vect (\matr{R}) - \matr{R}_{:, a} \succeq 0\\
    &\qquad  \Leftrightarrow \left[ \matr{I}_{N_S} + \gamma (\matr{T}^\pi - \matr{T}_{:, a, :})(\matr{I}_{N_S} - \gamma \matr{T}^\pi)^{-1}\right] \hat{\matr{\pi}}\vect (\matr{R})  - \matr{R}_{:, a} \succeq 0 \\
    &\qquad  \Leftrightarrow \left[ \widetilde{\matr{I}}_{N_S} + \gamma(\widetilde{\matr{T}}^\pi - \matr{T})(\matr{I}_{N_S} - \gamma \matr{T}^\pi)^{-1} \right]\hat{\matr{\pi}} \vect (\matr{R})  - \matr{R} \succeq 0
  \end{align*}
\end{proof}

\begin{theorem}
  \label{thm:r_valid_equiv}
  Maximising expression (\ref{equ:old_objective_function}) is equivalent to minimising
  \begin{equation}
    \label{equ:objective_function}
    \sum_{s \in S} \max_{a \in A \backslash\{\pi(s)\}}\left( \hat{\matr{I}} - \left(\widetilde{\matr{I}}_{N_S} - \gamma \matr{T}\right)\left(\matr{I}_{N_S} - \gamma \matr{T}^\pi\right)^{-1} \hat{\matr{\pi}}\right)_{s, a} \cdot \vect(\matr{R}) + \lambda \vert \vert \matr{R} \vert \vert_1
  \end{equation}
\end{theorem}

\begin{proof}
  Maximising expression (\ref{equ:old_objective_function}) is equivalent to minimising
  \begin{align*}
    &\sum_{s \in S}\left\{ \max_{a \in A \backslash\{\pi(s)\}}\left(\matr{Q}^\pi_{s, a}\right) - \matr{Q}^\pi_{s, \pi(s)} \right\} + \lambda \vert \vert \matr{R} \vert \vert_1 \\
    & \qquad = \sum_{s \in S} \max_{a \in A \backslash\{\pi(s)\}}\left( \matr{Q}^\pi_{s, a} - \matr{V}^\pi_s \right) + \lambda \vert \vert \matr{R} \vert \vert_1 \\
    &\qquad  = \sum_{s \in S} \max_{a \in A \backslash\{\pi(s)\}} \left( \matr{Q}^\pi -  \widetilde{\matr{I}}_{N_S} \matr{V}\right)_{s, a} + \lambda \vert \vert \matr{R}\vert \vert_1 \\
    & \qquad = \sum_{s \in S} \max_{a \in A \backslash\{\pi(s)\}}\left\{ \left( \hat{\matr{I}} - \left( \widetilde{\matr{I}}_{N_S} - \gamma \matr{T}\right)\left(\matr{I}_{N_S} - \gamma \matr{T}^\pi\right)^{-1} \hat{\matr{\pi}}\right)_{s, a} \cdot \vect(\matr{R}) \right\} \\
    & \qquad \qquad + \lambda \vert \vert \matr{R} \vert \vert_1
  \end{align*}
\end{proof}

In the tabular setting, we can express \cref{equ:r_state_plus_action} as
$$
\vect(\matr{R}) = \vect(\matr{R}_A) + \hat{\mathbb{1}}_{N_s,N_A}^\intercal \matr{R}_S
$$
Where $\hat{\mathbb{1}}_{N_s,N_A}^\intercal$ is such that $ \hat{\mathbb{1}}_{N_s, N_A}^\intercal \matr{R}_S $ equal to $N_A$ copies of $\matr{R}_S$ stacked on top of each other.

\begin{theorem}
  \label{thm:express_in_rs}
  Writing $\matr{\Omega} =  \hat{\matr{I}} - \left( \widetilde{\matr{I}}_{N_S} - \gamma \matr{T}\right)\left(\matr{I}_{N_S} - \gamma \matr{T}^\pi\right)^{-1} \hat{\matr{\pi}}$, minimising expression (\ref{equ:objective_function}) is equivalent to minimising
  \begin{equation}
    \sum_{s \in S} \max_{a \in A \backslash \{\pi(s)\}} \matr{\Omega}_{s, a} \hat{\mathbb{1}}_{N_s,N_A}^\intercal \matr{R}_S + \lambda \left\vert \left\vert \matr{R}_S \right\vert \right\vert_1
  \end{equation}
\end{theorem}

\begin{proof}
  Expression (\ref{equ:objective_function}) is
  \begin{align*}
    & \sum_{s \in S} \max_{a \in A \backslash \{ \pi(s) \}} \matr{\Omega}_{s, a} \vect(\matr{R}) + \lambda \vert\vert \matr{R} \vert \vert_1\\
    & \qquad = \sum_{s \in S} \max_{a \in A \backslash \{ \pi(s) \}} \matr{\Omega}_{s, a} \left(\vect(\matr{R}_A) + \onehat \matr{R}_S\right) \\
    & \qquad \qquad + \lambda \left\vert\left\vert \vect(\matr{R}_A) + \hat{\mathbb{1}}_{N_s,N_A}^\intercal \matr{R}_S \right\vert\right\vert_1 \\
  \end{align*}
  As the second term is an arbitrary choice to encourage more zeros in the reward, we can replace it with $\vert\vert \matr{R}_S\vert\vert_1$. We therefore get that we want to minimise
  $$
  \sum_{s} \left\{\max_{a \in A \backslash\{\pi(s)\}} \left(\matr{\Omega}_{s, a} \vect(\matr{R}_A)\right) +\max_{a \in A \backslash\{\pi(s)\}} \left(\matr{\Omega}_{s, a} \onehat \matr{R}_S \right)\right\} + \lambda \vert\vert\matr{R}_S\vert\vert_1
  $$
  And since $\matr{R}_A$ is fixed, we get
  \begin{equation*}
    \sum_{s \in S} \max_{a \in A \backslash \{\pi(s)\}} \matr{\Omega}_{s, a} \hat{\mathbb{1}}_{N_s,N_A}^\intercal \matr{R}_S + \lambda \left\vert \left\vert \matr{R}_S \right\vert \right\vert_1
  \end{equation*}
\end{proof}

\end{document}